\documentclass[10pt,twocolumn,letterpaper]{article}

\usepackage[pagenumbers]{cvpr} % To force page numbers, e.g. for an arXiv version

\definecolor{cvprblue}{rgb}{0.21,0.49,0.74}
\usepackage[pagebackref,breaklinks,colorlinks,allcolors=cvprblue]{hyperref}
\usepackage{url}

\usepackage{booktabs}       % professional-quality tables
\usepackage{amsfonts}       % blackboard math symbols
\usepackage{nicefrac}       % compact symbols for 1/2, etc.
\usepackage{microtype}      % microtypography
\usepackage{xcolor}         % colors
\usepackage{natbib}         % references
\usepackage{makecell}       % table
\usepackage{multirow}       % table
\usepackage{amsmath}        % math
\usepackage{mathtools}
\usepackage{graphicx}       % figure
\usepackage{subcaption}  % Required for creating subfigures
\usepackage{color}
\usepackage[normalem]{ulem}
\usepackage{tikz}
\usepackage{bbding}
\usepackage{pifont}
\usepackage{colortbl}
\usepackage{bm}
\usepackage{xspace}

\usepackage{ragged2e}
\usepackage{amssymb}
\usepackage{rotating}
\usepackage{array}
\usepackage{times}
\usepackage{psfrag}
\usepackage{enumitem}
\usepackage{lscape}
\usepackage{verbatim}
\usepackage{overpic}
\usepackage{amsthm}
\usepackage{tablefootnote}
\usepackage{hhline}

\usepackage[ruled,linesnumbered]{algorithm2e}
\SetKwComment{Comment}{// }{}

\newtheorem{theorem}{Theorem}

\usepackage[misc]{ifsym}

\DeclareMathOperator{\Var}{Var}
\DeclareMathOperator{\Cov}{Cov}

\def\paperID{***} % *** Enter the Paper ID here
\def\confName{CVPR}
\def\confYear{2026}

\title{BinaryAttention: One-Bit QK-Attention for Vision and Diffusion Transformers}

\author{Chaodong Xiao$^{1,2}$, Zhengqiang Zhang$^{1,2}$, Lei Zhang$^{1,2,}$\thanks{Corresponding author. This research is supported by the PolyU-OPPO Joint Innovative Research Center.}\\
{$^{1}$The Hong Kong Polytechnic University \qquad $^{2}$OPPO Research Institute} \\
{\tt\small chaodong.xiao@connect.polyu.hk, cslzhang@comp.polyu.edu.hk}
}

\begin{document}
\maketitle

\begin{abstract}
Transformers have achieved widespread and remarkable success, while the computational complexity of their attention modules remains a major bottleneck for vision tasks. 
Existing methods mainly employ 8-bit or 4-bit quantization to balance efficiency and accuracy. 
In this paper, with theoretical justification, we indicate that binarization of attention preserves the essential similarity relationships, and propose \textbf{BinaryAttention}, an effective method for fast and accurate 1-bit qk-attention. 
Specifically, we retain only the sign of queries and keys in computing the attention, and replace the floating dot products with bit-wise operations, significantly reducing the computational cost. 
We mitigate the inherent information loss under 1-bit quantization by incorporating a learnable bias, and enable end-to-end acceleration. 
To maintain the accuracy of attention, we adopt quantization-aware training and self-distillation techniques, mitigating quantization errors while ensuring sign-aligned similarity. 
BinaryAttention is more than \textbf{$2\times$ faster} than FlashAttention2 on A100 GPUs.
Extensive experiments on vision transformer and diffusion transformer benchmarks demonstrate that BinaryAttention matches or even exceeds full-precision attention, validating its effectiveness. 
Our work provides a highly efficient and effective alternative to full-precision attention, pushing the frontier of low-bit vision and diffusion transformers. 
The codes and models can be found at \url{https://github.com/EdwardChasel/BinaryAttention}.

\end{abstract}

\section{Introduction}
\label{sec:intro}

Transformers \cite{vaswani2017attention} have made great breakthroughs in different fields, from natural language processing \cite{devlin2019bert, raffel2020exploring, brown2020language, chowdhery2023palm, touvron2023llama, team2024gemini, guo2025deepseek}, to vision tasks \cite{dosovitskiy2021an, liu2021swin, cheng2022masked, he2022masked, peebles2023scalable, ravi2024sam, tao2025instantcharacter}, and to multimodal fundamental models \cite{radford2021learning, li2023blip, liu2023visual, achiam2023gpt, li2024llava, wang2024qwen2}, largely due to the expressivity of attention mechanism. Despite the remarkable progress, this success comes with a cost: standard attention scales quadratically with sequence length, creating  extraordinary computational resources demand for long-context and high-resolution tasks. 
To alleviate this bottleneck, significant efforts \cite{qi2021accelerating, tay2022efficienttransformerssurvey, tang2024survey, gholami2022survey, dao2022flashattention, gu2024mamba} have been dedicated to accelerating Transformers. These approaches can be broadly grouped into the \textit{architecture optimization}, \textit{model quantization}, and \textit{hardware optimization} categories.%, as illustrated in Fig. 1.

\begin{figure}[!t]
    \centering
    \includegraphics[width=\columnwidth]{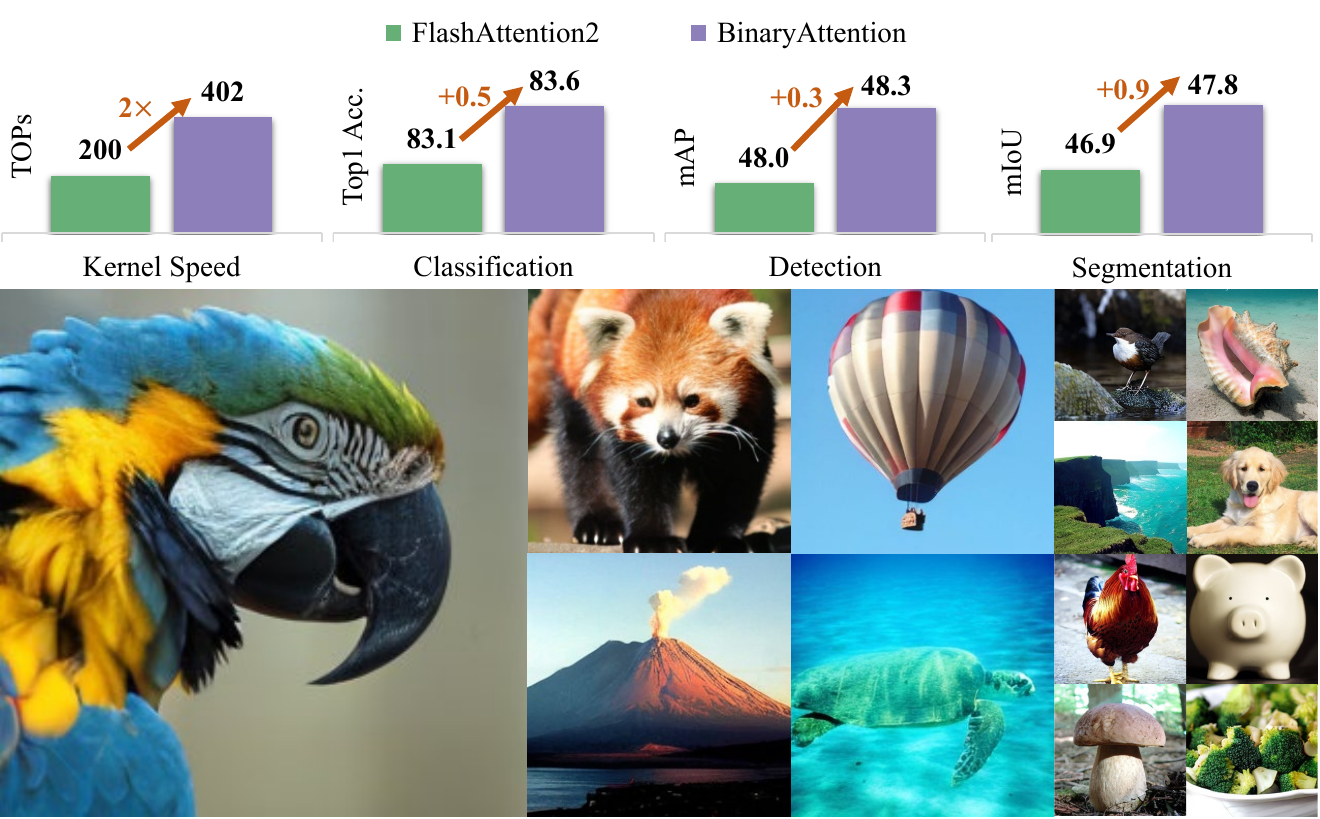}
    % \vspace{-7mm}
    \caption{\textbf{Top}: Performance comparison between FlashAttention2 and BinaryAttention on vision tasks. \textbf{Bottom}: Image generation examples by DiT-XL/2 \cite{peebles2023scalable} driven by BinaryAttention.}
    \label{fig:samples}
    \vspace{-4mm}
\end{figure}

\textit{Architecture optimization} seeks to reduce computational overhead by revising the attention mechanism. Linear attention \cite{katharopoulos2020transformers, wang2020linformer, yu2022metaformer, arora2024simple, yang2024parallelizing} replaces quadratic dot-product operations with linear complexity kernel computations. Sparse attention \cite{beltagy2020longformer, roy2021efficient, tay2020sparse, xiao2023efficient, gao2024seerattention, yuan2025native} restricts computations to a selected subset of token pairs. More recent state space models \cite{gu2022efficiently, gu2024mamba, mamba2} substitute the standard attention with a recurrent data-dependent selection mechanism. While effective, these methods often struggle to maintain the expressive power of standard attention across diverse models and tasks.

\textit{Model quantization} \cite{jacob2018quantization, chen2020statistical, gholami2022survey} is a principled approach to accelerate training and inference while shrinking memory by reducing numerical precision. The quantization of linear layers has been explored in depth \cite{yang2019quantization, liu2021post, li2022q, frantar2022gptq, liu2023llm, lin2024awq, he2023bivit, le2023binaryvit, xiao2024binaryvit} and is relatively mature. In contrast, recent efforts such as SageAttention  \cite{zhang2025sageattention, zhang2025sageattention2, zhang2025sageattention3} have increasingly focused on the quantization of attention. Unlike linear layers, quantizing attention presents unique challenges due to the dynamic nature and the sensitive softmax normalization. Consequently, these approaches typically employ 8-bit or 4-bit representations (\eg, INT8, FP8, INT4 and FP4) to maintain a practical balance between efficiency and accuracy. However, further reduction of the precision to sub-4-bit levels, especially to binary representations, remains a major hurdle, as the extreme information loss and optimization instability cause an abrupt performance degradation.

\textit{Hardware optimization} leverages specialized hardware designs and kernel optimizations to accelerate Transformers. Breakthroughs such as FlashAttention  \cite{dao2022flashattention, dao2023flashattention2, shah2024flashattention} achieve significant speedups on GPUs without altering the model architecture or sacrificing accuracy. Despite their effectiveness, an alternative optimization focuses on extreme low-precision attention computation. Specifically, matrix multiplications with entries in binary representations can be efficiently implemented on modern hardware \cite{hubara2016binarized}. Therefore, developing an effective and hardware-friendly binary attention mechanism is an imperative demand to push the boundaries of the efficiency of Transformers.

In this paper, we theoretically analyze the feasibility of binary representations in attention computing and indicate that the essential similarity relationships can be preserved even in binary space. Building upon this insight, we present \textbf{BinaryAttention}, a novel quantization method that enables fast and accurate 1-bit qk-attention. We demonstrate the significant potential of 1-bit qk-attention for vision and diffusion transformers, achieving remarkable acceleration without compromising performance, as shown in Fig.~\ref{fig:samples}. 
Specifically, we quantize attention queries and keys to 1-bit representations. This transforms the standard dot-product attention score into a distance-based and a direction-based mechanism, which can be computed with hyper-efficient bit-wise XNOR and popcount instructions. While this can drastically reduce the computational cost, relying solely on the similarity in binary space can cause the attention distribution to become overly uniform or flattened, as it discards the crucial magnitude information from the original tokens. To mitigate this flattened effect, we then introduce a learnable bias term, which is designed to be dense, position-sensitive, or context-aware, allowing expressive and discriminative 1-bit qk-attention. Furthermore, we design a hybrid quantization scheme that applies 8-bit precision to the attention weights and values, enabling end-to-end acceleration. 
Finally, to address the inherent approximation errors and the distribution shift caused by 1-bit quantization, we employ quantization-aware training (QAT) \cite{jacob2018quantization} and self-distillation \cite{hinton2015distilling} techniques to guide the model learning binary representations whose similarity aligns closely with their full-precision counterparts.

By adapting the hardware acceleration method of FlashAttention2 \cite{dao2023flashattention2} into our BinaryAttention kernel, we achieve a more than 100\% inference speedup over FlashAttention2 on A100 GPUs. We perform a comprehensive evaluation of BinaryAttention across vision transformers and diffusion transformers on fundamental vision tasks such as image classification, detection, segmentation, and image generation. The results demonstrate that BinaryAttention consistently  achieves or even exceeds the performance of its full-precision attention counterparts. Our work establishes a significantly efficient and effective alternative to full-precision attention, largely advancing the development of efficient transformers for visual tasks.

\section{Related Work}
\label{sec:related work}

\textbf{Attention architecture.} In response to the quadratic complexity of attention in Transformers, there have been significant attempts in redesigning the computation architecture of attentions, including linear attention \cite{katharopoulos2020transformers, wang2020linformer, yu2022metaformer, arora2024simple, yang2024parallelizing}, sparse attention \cite{beltagy2020longformer, roy2021efficient, tay2020sparse, xiao2023efficient, gao2024seerattention, yuan2025native}, and state space models \cite{gu2022efficiently, gu2024mamba, mamba2}. 
Linear attention reformulates the computing process of attention to achieve linear complexity. For instance, Katharopoulos \etal \cite{katharopoulos2020transformers} replaced the softmax operation with carefully designed kernel functions. Wang \etal \cite{wang2020linformer} proposed Linformer, which adopts an alternative design that approximates the attention computation using low-rank matrix factorization. Yang \etal \cite{yang2024parallelizing} generalized linear attention based on the gated delta rule, allowing more expressive variants while maintaining linear complexity. Sparse attention approaches reduce the complexity by limiting interactions to strategically selected token pairs, encompassing common designs such as sliding window \cite{beltagy2020longformer, roy2021efficient}, sink \cite{tay2020sparse, xiao2023efficient} and hybrid \cite{gao2024seerattention,yuan2025native} patterns. Recently, state space models (SSMs) \cite{gu2021combining, nguyen2022s4nd, gu2024mamba, mamba2, zhu2024vim} have emerged to replace attention with efficient recurrent scanning processes. Models like Mamba \cite{gu2024mamba, mamba2} demonstrated that SSMs achieve Transformer-like capability with linear complexity.

\textbf{Quantization techniques.} Quantization techniques reduce model precision to achieve acceleration, which can be divided into two categories: post-training quantization (PTQ) \cite{liu2021post, frantar2022gptq, lin2024awq, zhang2025sageattention, zhang2025sageattention2, zhang2025sageattention3} minimizes quantization error without retraining, and quantization-aware training (QAT) \cite{yang2019quantization, li2022q, liu2023llm, he2023bivit, le2023binaryvit, xiao2024binaryvit} simulates the effects of quantization during the training or fine-tuning process. Frantar \etal \cite{frantar2022gptq} presented GPTQ, which quantizes weight parameters by leveraging second-order statistics. Lin \etal \cite{lin2024awq} developed AWQ, which protects salient weights by determining the per-channel scaling factors based on the distribution of activations. Several approaches have pushed the quantization to binary space specifically for vision transformers \cite{he2023bivit, le2023binaryvit, xiao2024binaryvit}. More recent efforts have extended quantization to the computation of attention. Zhang \etal \cite{zhang2025sageattention} introduced SageAttention, which employs block-wise INT8 quantization for attention queries and keys. SageAttention2 \cite{zhang2025sageattention2} quantizes queries and keys to INT4 and computing the attention map and values in FP8. SageAttention3 \cite{zhang2025sageattention3} goes further by unifying the attention computations to FP4.

\textbf{Hardware optimization.} This strategy unleashes hardware capabilities to improve efficiency. Lefaudeux \etal \cite{xFormers2022} developed xFormers, which optimizes attention through customized and memory-efficient CUDA kernels. Dao \etal \cite{dao2022flashattention} introduced the concept of IO-aware tiled attention and proposed FlashAttention, achieving significant speedups. FlashAttention2 \cite{dao2023flashattention2} refines this through improved parallelism and warp-level partition. FlashAttention3 \cite{shah2024flashattention} further exploits Hopper GPU architecture by leveraging asynchronous communication and FP8 Tensor Core. Based on these developments, SageAttention \cite{zhang2025sageattention, zhang2025sageattention2, zhang2025sageattention3}  integrates FlashAttention with low-precision quantization, leading to significant gains in overall efficiency.

\section{Preliminaries}
\label{sec:preliminaries}

\textbf{Model quantization} \cite{gholami2022survey} aims to quantize the high-bit (\eg, FP32) weights and activations of a network model into low-bit representations (\eg, INT8) for efficient deployment. A simple approach is uniform quantization, which linearly maps the floating-point values to a discrete set of levels. For a given full-precision vector $\displaystyle \bm{x}\in \mathbb{R}^d$, this process can be described as: $\displaystyle \tilde{\bm{x}}=\lceil\bm{x}/s\rfloor+z$,  $\hat{\bm{x}}=s(\tilde{\bm{x}}-z)$,
where $\displaystyle \tilde{\bm{x}}$ is the quantized value, $\lceil\cdot\rfloor$ means rounding operation, $\displaystyle s$ is a scaling factor, $\displaystyle z$ is a zero point, and $\displaystyle \hat{\bm{x}}$ denotes the de-quantized value that approximates the original $\displaystyle \bm{x}$.  An extreme case is binary quantization \cite{pouransari2020least}, which simplifies the approximation to $\displaystyle \hat{\bm{x}}=\alpha \text{sign}(\bm{x})$ with an optimal scaling factor $\displaystyle \alpha$, where the element-wise function $\displaystyle \text{sign}(\bm{x})$ maps non-negative values to $1$ and others to $-1$.

\vspace{+2mm}
\noindent\textbf{Attention} is the cornerstone of the Transformer architecture. Given an input $\displaystyle \bm{x} \in \mathbb{R}^{N\times d}$ of length $\displaystyle N$, single head softmax attention \cite{vaswani2017attention} computes output $\displaystyle \bm{y} \in \mathbb{R}^{N\times d}$ as:
\begin{equation}
    \label{eq:standard attention}
    \bm{y}_i =\sum_{j=1}^N(\frac{\text{exp}(\bm{q}_i^T\bm{k}_j/\sqrt{d})}{\sum_{j=1}^N \text{exp}(\bm{q}_i^T\bm{k}_j/\sqrt{d})})\bm{v}_j=\sum_{j=1}^N\bm{P}_{ij}\bm{v}_j,
\end{equation}
where the query, key and value tokens $(\bm{q}_i,\bm{k}_j,\bm{v}_j\in\mathbb{R}^d)$ are generated by projecting $\bm{x}$ with learnable weight matrices. $\bm{P}_{ij}$ is the attention coefficient computed over $(\bm{q}_i,\bm{k}_j)$.

\section{BinaryAttention}
\label{method}

\begin{figure*}[!t]
    \centering
        \begin{subfigure}{0.49\textwidth}
        \centering
        \includegraphics[height=4.5cm]{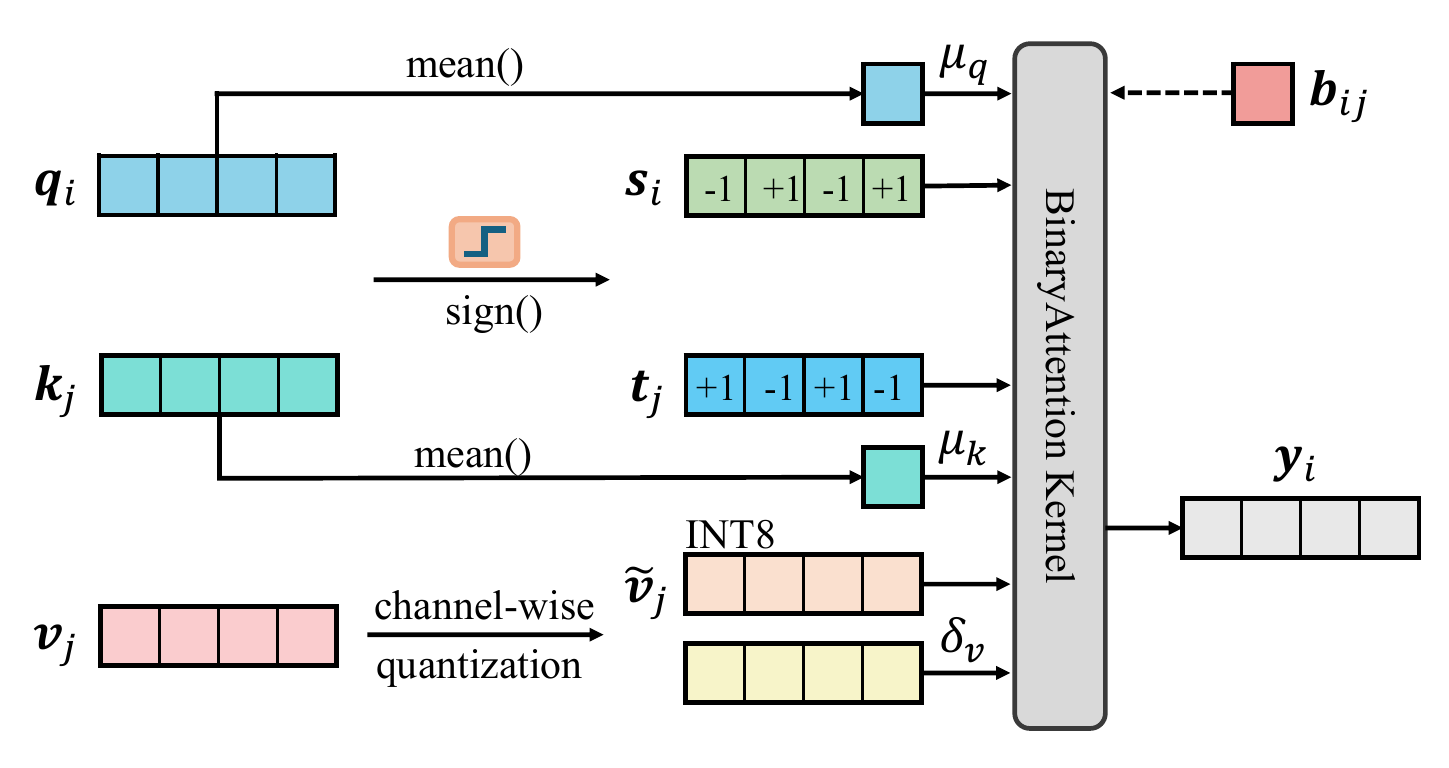}
        \caption{Overall of BinaryAttention}
        \label{fig:workflow (a)}
    \end{subfigure}
    \hfill
    \begin{subfigure}{0.24\textwidth}
        \centering
        \includegraphics[height=4.5cm]{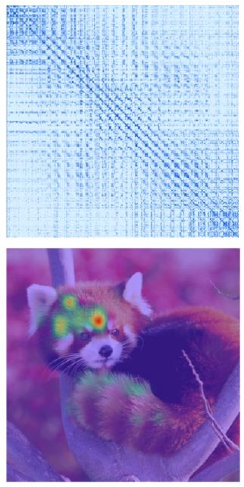}
        \caption{Standard Attention}
         \label{fig:workflow (b)}
    \end{subfigure}
    \hfill
    \begin{subfigure}{0.24\textwidth}
        \centering
        \includegraphics[height=4.5cm]{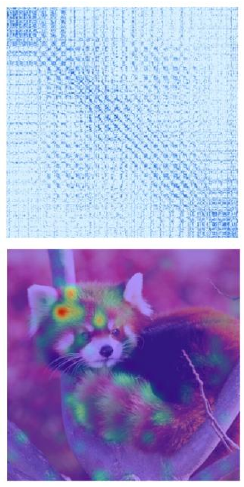}
        \caption{BinaryAttention}
        \label{fig:workflow (c)}
    \end{subfigure}
    % \vspace{-3mm}
    \caption{Overview and comparative analysis of BinaryAttention. 
    (a) The computation of BinaryAttention involves three components: converting queries and keys into scaled binary representations, applying a bias enhancement, and quantizing the attention coefficients and values.  
    Sub-figures (b) and (c) show the attention maps (top) and the corresponding activation maps (bottom) for Standard Attention and our BinaryAttention, demonstrating the comparable expressivity of BinaryAttention to Standard Attention despite 1-bit quantization.}
    \label{fig:workflow}
    \vspace{-2mm}
\end{figure*}

\subsection{Theoretical Motivation for BinaryAttention}
\label{sec:theoreticalmotivation}
We commence by establishing a theoretical bridge between standard and binary attention, demonstrating the feasibility of our approach. As shown in Eq.~(\ref{eq:standard attention}), softmax attention computes a weighted sum over the values. The attention coefficient $\bm{P}_{ij}$ is determined by the dot-product similarity $\bm{S}_{ij}=\bm{q}_i^T\bm{k}_j$ between query and key. This similarity can be interpreted from two complementary perspectives.

Suppose that the query $\bm{q}_i$ and key $\bm{k}_j$ are $L_2$ normalized, a common practice known as QKNorm \cite{henryetal2020query}. With $\displaystyle \|\bm{q}_i-\bm{k}_j\|_2^2=\|\bm{q}_i\|_2^2+\|\bm{k}_j\|_2^2-2\bm{q}_i^T\bm{k}_j$, the attention coefficient can be expressed as:
\begin{equation}
    \bm{P}_{ij}  \propto \text{exp}(-\|\bm{q}_i-\bm{k}_j\|_2^2/\tau),
\end{equation}
where $\displaystyle \tau$ is a scaling factor. This formulation reveals that softmax attention actually behaves as a distance-based metric in Euclidean space.
Alternatively, the dot-product similarity can be rewritten as $\displaystyle \bm{q}_i\bm{k}_j^T=\|\bm{q}_i\|_2\|\bm{k}_j\|_2\cos{\theta}$, where $\displaystyle \theta$ refers to the angle between the query $\displaystyle \bm{q}_i$ and key $\bm{k}_j$. With $L_2$ normalization, the attention coefficient is therefore proportional to the cosine similarity, scaled by $\displaystyle \tau$:
\begin{equation}
    \bm{P}_{ij}  \propto \text{exp}(\cos{\theta}/\tau).
\end{equation}
This shows that the attention mechanism can be interpreted as operating on the directional similarity.

With the above dual perspectives of standard attention in mind, we explore whether these relationships can be preserved under 1-bit quantization of it. 
We represent the binary counterparts of query $\bm{q}_i$ and key $\bm{k}_j$ as $\displaystyle \bm{s}_i=\text{sign}( \bm{q}_i), \bm{t}_j=\text{sign}( \bm{k}_j)\in\{-1,1\}^d$, respectively. Then, the dot-product similarity can be expressed directly in terms of the Hamming distance as $\displaystyle \bm{S}_{ij}=\bm{s}_i^T\bm{t}_j=d-2\|\bm{s}_i-\bm{t}_j\|_H$, 
where the Hamming norm $\displaystyle \|\bm{x}\|_H$ is defined as the number of non-zero entries of $\displaystyle \bm{x}$. Thus, the attention coefficient is given by:
\begin{equation}
    \bm{P}_{ij}  \propto \text{exp}(-\|\bm{s}_i-\bm{t}_j\|_H/\tau).
\end{equation}
This indicates that binary attention operates as a distance-based metric in the Hamming space, mirroring the Euclidean distance in standard attention. Beyond this, it also preserves the directional similarity. Since the dot-product $\bm{s}_i^T\bm{t}_j$ in binary domain equals $d\cos{\theta}$, the attention coefficient can be equivalently expressed as $\bm{P}_{ij} \propto \text{exp}(\cos{\theta}/\tau)$.

While the above analyzes reveal the structural parallels between binary and standard attentions in their respective spaces, a more fundamental connection exists at the statistical level. We show that binary attention preserves the covariance structure of the original queries and keys, as shown in the following \textbf{Theorem~\ref{theorem1}}.

\begin{theorem}\label{theorem1}
Consider two random variables $\bm{q},\bm{k}\in \mathbb{R}^d$. Suppose that $\bm{z}=(\bm{q}^T,\bm{k}^T)^T\in \mathbb{R}^{2d}$ is a zero-mean Gaussian vector with covariance matrix $\bm{\Sigma}$, where $\bm{\Sigma}=\left[ \begin{array}{cc} \bm{\Sigma}_{qq} & \bm{\Sigma}_{qk} \\ \bm{\Sigma}_{kq} & \bm{\Sigma}_{kk} \end{array} \right]$. Denote $\bm{D}_q=diag(\bm{\Sigma}_{qq}),\bm{D}_k=diag(\bm{\Sigma}_{kk})$ and $\bm{C}=\bm{D}_q^{-\frac{1}{2}}\bm{\Sigma}_{qk}\bm{D}_k^{-\frac{1}{2}}$. For any $\bm{s}=\text{sign}(\bm{q})$ and $\bm{t}=\text{sign}(\bm{k})$, there is:
$$
    \mathbb{E}[\bm{s}\bm{t}^T]=\frac{2}{\pi}\arcsin{\bm{C}}.
$$
\end{theorem}
\begin{proof}
Please see \textbf{supplementary file} for the proof.
\end{proof}

\textbf{Theorem~\ref{theorem1}} implies that the outer product formed by binary queries and keys is a consistent estimate of the original covariance matrix. Crucially, the covariance matrix shares equivalent non-zero eigenspectrum with the Gram matrix, which contains all pairwise dot products between queries and keys and governs the core relational structure of standard attention. Theorem~\ref{theorem1} provides a guarantee for the performance of binary attention, ensuring its expressive capability.

\subsection{Formulation of BinaryAttention}

We now propose BinaryAttention, an effective 1-bit qk-attention method. 
As illustrated in Fig.~\ref{fig:workflow (a)}, BinaryAttention comprises the following three key components.

\textbf{Scaled binary representations}. The primary computational bottleneck in standard attention lies in the floating-point matrix multiplication between queries and keys. To address this, we replace this expensive operation with a highly efficient binary alternative. Specifically, we first quantize the query $\bm{q}_i$ and key $\bm{k}_j$ into  binary values via a scaled 1-bit quantization function:
\begin{equation}\label{eq:binaryquant}
    \bm{s}_i=\mu_q\text{sign}(\bm{q}_i),\quad \bm{t}_j=\mu_k\text{sign}(\bm{k}_j),
\end{equation}
where $\displaystyle \mu_q,\mu_k\in \mathbb{R}_{\geq 0} $ are the means of queries and keys along the token and channel axes, respectively. The dot-product similarity $\displaystyle \bm{S}_{ij}$ is computed as $\displaystyle \mu_q\mu_k\bm{s}_i^T\bm{t}_j$ in binary space, which is not only statistically aligned with its full-precision counterpart but also extremely efficient by leveraging bit-wise XNOR and popcount instructions. 

\textbf{Bias enhancement}.
While computationally efficient, the 1-bit quantization of queries and keys introduces two challenges that can degrade attention performance: a loss of magnitude information and a subsequent distribution shift in the attention scores. The binary representations project data of varying scales onto the unit hyper-sphere, sacrificing the nuanced relationships captured by the original full-precision dot products. Consequently, the softmax distribution in binary space tends to produce overly uniform attention coefficients that are difficult to distinguish salient features. To counteract this, we introduce a bias term as follows:
\begin{equation}
    \bm{S}_{ij}=\mu_q\mu_k\bm{s}_i^T\bm{t}_j/\sqrt{d} + \bm{b}_{ij},
\end{equation}
where $\displaystyle \bm{b}_{ij}\in\mathbb{R}$ is an optional bias that can be tailored to different scenarios and architectures. For instance, the bias can be a dense learnable matrix to increase the rank of dot-product similarities in binary space. Alternatively, it can be instantiated as position-sensitive or context-aware to explicitly model spatial structural and context-specific priors.

The bias term acts as a corrective measure, reintroducing the contextual or structural information back into the computation of attention coefficients, effectively avoiding the collapse of the attention distribution and enabling BinaryAttention to capture complex and long-range dependencies that are critical for visual tasks.

\textbf{Quantization of attention coefficients and values}.
To achieve a holistic acceleration of the entire attention computation, we further extend the low-bit computation to the attention coefficients and the values, which are primarily memory-bound. To address this, BinaryAttention pairs the dot-product similarities in binary space with specific 8-bit quantization schemes tailored to these two components. For the attention coefficients $\bm{P}_{ij}$, which is naturally constrained to the range $[0,1]$ by the softmax operation, we employ an unsigned 8-bit quantization with a static scale of $1/255$. For the value $\bm{v}_j$, which typically exhibits more complex statistical distributions with potential outliers across channels, we adopt a channel-wise 8-bit quantization strategy with a scale $\delta_v$. The quantization process and subsequent value aggregation are formulated as:
% \vspace{-4pt}
\begin{equation}\label{eq:int8quant}
    \tilde{\bm{P}}_{ij}=\lceil\bm{P}_{ij}\times 255 \rfloor, \ \tilde{\bm{v}}_j=\lceil\bm{v}_j/\delta_v \rfloor, \ \bm{y}_i=\sum_{j=1}^N \frac{\delta_v}{255}\tilde{\bm{P}}_{ij}\tilde{\bm{v}}_j.
\end{equation}
Here, $\displaystyle \tilde{\bm{P}}_{ij}$ and  $\tilde{\bm{v}}_j$ denote the quantized 8-bit integers. This design enables efficient integer operations while maintaining the accuracy through proper scaling factors.

\noindent \textbf{Remark.} The BinaryAttention framework, through its three core components, achieves a balance between computational efficiency and representation capabilities. The scaled binary representations enable hardware-friendly computation, the optional bias recovers discriminative ability, and the hybrid quantization ensures end-to-end acceleration. As shown in Fig.~\ref{fig:workflow (b)} and \ref{fig:workflow (c)}, the attention and activation maps from BinaryAttention reveal strikingly similar patterns with those from standard attention, focusing on analogous related regions with a high correlation. This alignment demonstrates that, even under extreme 1-bit quantization, BinaryAttention retains the content-based dynamic routing and long-range dependencies modeling capabilities of standard attention.

\subsection{Hardware-Aware Implementation}
\label{hardware}

We take advantage of the capabilities of modern GPU hardware to implement BinaryAttention, building upon the foundational principles of FlashAttention2 \cite{dao2023flashattention2} while introducing dedicated optimizations. The complete algorithm is detailed as Algorithm 1 in \textbf{supplementary file}. In particular, we utilize the fast \texttt{mma.s32.b1.b1.s32} PTX instruction (referred to as `BinaryMatmul' in Algorithm 1) of NVIDIA Tensor Cores for the similarity computations between binary queries and keys. For the attention coefficients and values multiplications, we employ the \texttt{mma.s32.u8.s8.s32} instruction (referred to as `IntMatmul' in Algorithm 1), which is optimized for mixed-precision 8-bit matrix operations. 

Our implementation maintains the memory hierarchy optimizations and block tiling strategies of FlashAttention2 \cite{dao2023flashattention2}, but adapts them specifically for the binary and low-precision context. This hardware-aware design ensures that BinaryAttention delivers practical speedups through extreme quantization of attention computation. 

\section{Experimental Results}

We conduct extensive experiments to evaluate the efficiency and effectiveness of BinaryAttention. 
Our primary competitors are FlashAttention2 \cite{dao2023flashattention2} and SageAttention \cite{zhang2025sageattention}, which share the same objective as BinaryAttention, \ie, accelerating the computation of standard attention. 
While linear attention \cite{hydra2023attn, shen2021efficient, you2022castling, cai2022efficientvit, han2023flatten, han2024bridging} and SSMs \cite{nguyen2022s4nd, zhu2024vim} also improve efficiency, they address the quadratic complexity problem by significantly changing the computing architecture of attention, which are actually orthogonal to our work. In addition, traditional model quantization methods \cite{yuan2022ptq4vit, li2023vit} mainly target linear layers rather than attention computations. Therefore, they are also complementary to our work.

In the following experiments, we take FlashAttention2 as the baseline to implement our method. 
% Due to the page limit, the ablation studies on the key components of BinaryAttention are presented in the \textbf{supplementary file}.

\begin{figure}[!t]
    % \centering
    \includegraphics[width=0.91\columnwidth]{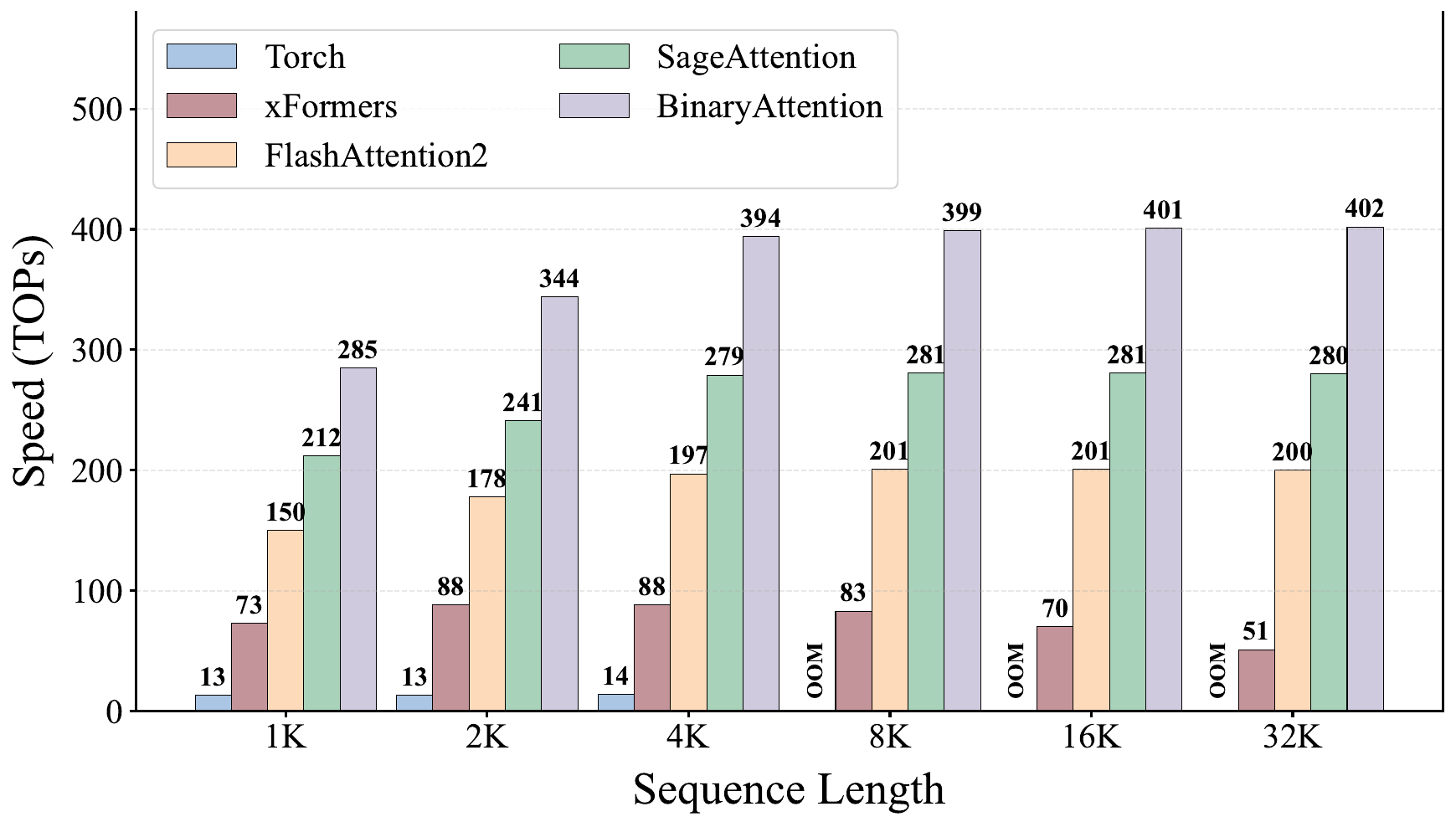}
    \vspace{-2mm}
    \caption{Kernel speed comparison on A100 GPUs. }
    \vspace{-2mm}
    \label{fig:speed comparision}
\end{figure}

\begin{figure}[!t]
    \centering
    \includegraphics[width=\columnwidth]{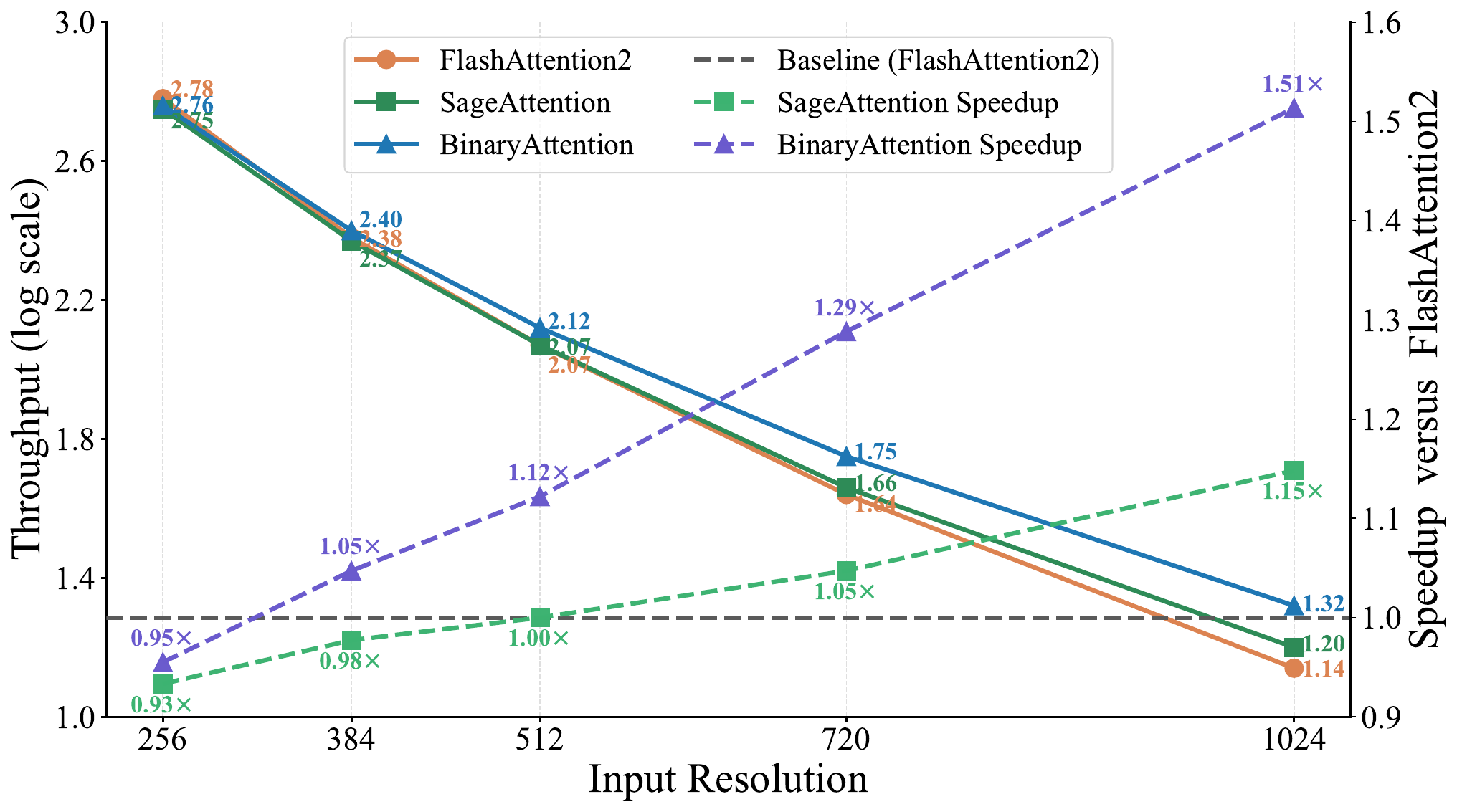}
    \vspace{-4mm}
    \caption{End-to-end throughput and speedup comparisons on A100 GPUs. ViT \cite{dosovitskiy2021an} models are used.}
    \label{fig:end to end}
    \vspace{-4mm}
\end{figure}

\subsection{Efficiency Comparison}
\label{sec:efficiency}

The two matrix multiplications dominate standard attention computation: the interaction between queries and keys  ($\displaystyle \bm{Q}\bm{K}^T$) and the aggregation of values ($\displaystyle \bm{P}\bm{V}$). The acceleration of BinaryAttention comes from the massive throughput advantages of low-precision arithmetic on modern hardware. Specifically, NVIDIA A100 Tensor Cores offer theoretical throughput of 312 TFLOPs/s for FP16, 624 TFLOPs/s for INT8, and an impressive 4992 TOPs/s for binary operations. By implementing $\displaystyle \bm{Q}\bm{K}^T$ with binary representations and $\displaystyle \bm{P}\bm{V}$ with INT8 quantization, BinaryAttention achieves theoretical speedups of 16$\times$ and 2$\times$ for these computations, respectively, yielding an overall theoretically 3.5$\times$ improvement for standard FP16 attention implementation. 

In practice, acceleration performance is constrained by memory bandwidth and I/O limitations. 
We benchmark the efficiency of BinaryAttention with several attention implementations, including Torch \cite{paszke2019pytorch}, xFormers \cite{xFormers2022}, FlashAttention2 \cite{dao2023flashattention2}, and SageAttention \cite{zhang2025sageattention}. 
Fig.~\ref{fig:speed comparision} shows the kernel speed across varying sequence lengths with a head dimension of 128 on A100 GPUs. It can be seen that BinaryAttention consistently outperforms all competing methods. Specifically, it is about 2$\times$ and 1.4$\times$ faster than FlashAttention2 and SageAttention, respectively. 

Moreover, we perform end-to-end inference and evaluate the log-scaled throughput on ViT \cite{dosovitskiy2021an} models with a patch size of 8, as presented in Fig.~\ref{fig:end to end}. While our method and SageAttention slightly slower than FlashAttention2 at low resolutions due to quantization overheads, BinaryAttention significantly outperforms all methods at higher resolutions, achieving a 1.5$\times$ speedup over FlashAttention2 and 1.3$\times$ over SageAttention at 1024$\times$1024 inputs. This further demonstrates that our BinaryAttention is compatible to modern hardware to achieve higher efficiency.

\subsection{Image Classification}

\textbf{Settings.} We first evaluate the image classification task using ImageNet-1K  \cite{deng2009imagenet}. Following DeiT \cite{touvron2021training}, we develop three variants of BinaryAttention, namely -T (tiny), -S (small) and -B (base), by substituting all standard attention modules with BinaryAttention. We follow the experimental settings in DeiT \cite{touvron2021training}, which are detailed in \textbf{supplementary file}. The models are fine-tuned with the self-distillation \cite{hinton2015distilling} strategy, where the full-precision counterparts serve as the teacher. 

We compare with quantization based methods PTQ4ViT \cite{yuan2022ptq4vit} (W8A8), I-ViT \cite{li2023vit} (W8A8) and the attention quantization method SageAttention \cite{zhang2025sageattention}. We also compare with Linear Transformers \cite{hydra2023attn, shen2021efficient, you2022castling, cai2022efficientvit, han2023flatten, han2024bridging} and SSMs \cite{nguyen2022s4nd, zhu2024vim} for reference. Following \cite{rastegari2016xnor, he2023bivit}, we count binary operations (BOPs) and floating-point operations (FLOPs) separately and report the total operations (OPs) as OPs$=$BOPs$/$64$+$FLOPs. 

\begin{table}[!t]
    \centering
    \setlength{\tabcolsep}{1.2mm}
    \resizebox{0.9\linewidth}{!}{
        \begin{tabular}{l|llccl}
        \toprule
         & Method & Reso. & \#Param. & OPs  &  Top-1 \\

        \midrule

        \multirow{8}{*}{\rotatebox{90}{Linear Transformer}}
        & Hydra-T \cite{hydra2023attn}          & $\text{224}^\text{2}$ & 6M  & 1.1G  & 68.3 \\
        & Efficient-T \cite{shen2021efficient}  & $\text{224}^\text{2}$ & 6M  & 1.1G  & 70.2 \\
        & Angular-T \cite{you2022castling}      & $\text{224}^\text{2}$ & 6M  & 1.1G  & 70.8 \\
        & Enhanced-T \cite{cai2022efficientvit} & $\text{224}^\text{2}$ & 6M  & 1.1G  & 72.9 \\

        & FLatten-T \cite{han2023flatten}       & $\text{224}^\text{2}$ & 6M  & 1.1G  & 74.1 \\
        \cline{2-6}

        & InLine-T \cite{han2024bridging}       & $\text{224}^\text{2}$ & 7M  & 1.1G  & 74.5 \\
        & InLine-S                              & $\text{288}^\text{2}$ & 17M & 5.0G  & 80.2 \\
        & InLine-B                              & $\text{448}^\text{2}$ & 24M & 17.2G & 82.3 \\
        % \cline{2-6}
        \midrule

        \multirow{4}{*}{\rotatebox{90}{SSM}}
        & S4ND-ViT-B \cite{nguyen2022s4nd}     & $\text{224}^\text{2}$ & 89M & -     & 80.4 \\
        & Vim-T \cite{zhu2024vim}              & $\text{224}^\text{2}$ & 7M  & 1.5G  & 76.1 \\
        & Vim-S                                & $\text{224}^\text{2}$ & 26M & 5.1G  & 80.3 \\
        & Vim-B                                & $\text{224}^\text{2}$ & 98M & 18.9G & 81.9 \\
        
        \bottomrule
        \toprule
        
        \multirow{21}{*}{\rotatebox{90}{Transformer}}
        
        & DeiT-T \cite{touvron2021training} & $\text{224}^\text{2}$ & 6M  & 1.2G  & 72.2 \\
        & DeiT-S                            & $\text{224}^\text{2}$ & 22M & 4.6G  & 79.8 \\
        & DeiT-B                            & $\text{224}^\text{2}$ & 87M & 17.6G & 81.8 \\
        & DeiT-B                            & $\text{384}^\text{2}$ & 87M & 55.4G & 83.1 \\
        \cline{2-6}

        & \multicolumn{5}{c}{W8A8 Quantization} \\
        & PTQ4ViT-T \cite{yuan2022ptq4vit} & $\text{224}^\text{2}$ & 6M  & 0.3G  & 71.6 \\
        & PTQ4ViT-S                        & $\text{224}^\text{2}$ & 22M & 1.2G  & 79.5 \\
        & PTQ4ViT-B                        & $\text{224}^\text{2}$ & 87M & 4.5G  & 81.5 \\
        & PTQ4ViT-B                        & $\text{384}^\text{2}$ & 87M & 14.2G & 83.0 \\
        \cline{2-6}

        & I-ViT-T \cite{li2023vit}        & $\text{224}^\text{2}$ & 6M & 0.3G   &  72.2 \\
        & I-ViT-S                         & $\text{224}^\text{2}$ & 22M & 1.2G  &  80.1 \\
        & I-ViT-B                         & $\text{224}^\text{2}$ & 87M & 4.5G  &  81.7 \\
        \cline{2-6}
        & \multicolumn{5}{c}{Attention Quantization} \\
        % \midrule
        & SageAttention-T \cite{zhang2025sageattention} & $\text{224}^\text{2}$ & 6M  & 1.2G  & 72.11 \\
        &\;+PTQ4ViT                                     & $\text{224}^\text{2}$ & 6M  & 0.4G  & 71.63 \\
        & SageAttention-S                               & $\text{224}^\text{2}$ & 22M & 4.5G  & 79.82 \\
        &\;+PTQ4ViT                                     & $\text{224}^\text{2}$ & 22M & 1.3G  & 79.38 \\
        & SageAttention-B                               & $\text{224}^\text{2}$ & 87M & 17.3G & 81.83 \\
        &\;+PTQ4ViT                                     & $\text{224}^\text{2}$ & 87M & 4.8G  & 81.58 \\
        & SageAttention-B                               & $\text{384}^\text{2}$ & 87M & 53.2G & 82.89 \\
        &\;+PTQ4ViT                                     & $\text{384}^\text{2}$ & 87M & 16.5G & 82.98  \\
        \hhline{~|-----}

        \rowcolor{blue!5} & BinaryAttention-T  & $\text{224}^\text{2}$ & 6M  & 1.1G & \textbf{72.88} \\
        \rowcolor{blue!5} &\;+PTQ4ViT          & $\text{224}^\text{2}$ & 6M  & 0.3G &          72.61 \\
        \rowcolor{blue!5} & BinaryAttention-S  & $\text{224}^\text{2}$ & 22M & 4.3G & \textbf{80.24} \\
        \rowcolor{blue!5} &\;+PTQ4ViT          & $\text{224}^\text{2}$ & 22M & 1.2G &         79.81  \\
        \rowcolor{blue!5} & BinaryAttention-B  & $\text{224}^\text{2}$ & 87M & 17.0G& \textbf{82.04} \\
        \rowcolor{blue!5} &\;+PTQ4ViT          & $\text{224}^\text{2}$ & 87M & 4.4G &         81.91  \\
        \rowcolor{blue!5} & BinaryAttention-B  & $\text{384}^\text{2}$ & 87M & 50.2G& \textbf{83.64} \\
        \rowcolor{blue!5} &\;+PTQ4ViT          & $\text{384}^\text{2}$ & 87M & 13.5G &        83.55  \\
        \bottomrule
        \end{tabular}
    }
    \caption{Comparison of image classification on ImageNet-1K.}
    \vspace{-4mm}
    \label{tab:classification}
\end{table}

\begin{table*}[!t]
    \centering
    \resizebox{0.9\linewidth}{!}{
    \setlength{\tabcolsep}{1.2mm}{
    \renewcommand\arraystretch{1.1}
        \begin{tabular}{l|cccccc|cccccc|cc}
            \toprule
            \multicolumn{15}{c}{\textbf{(a) Mask R-CNN}} \\
            Backbone & AP$^\text{b}$ & AP$^\text{b}_\text{50}$ & AP$^\text{b}_\text{75}$ & AP$^\text{b}_\text{s}$ & AP$^\text{b}_\text{m}$ & AP$^\text{b}_\text{l}$ & AP$^\text{m}$ & AP$^\text{m}_\text{50}$ & AP$^\text{m}_\text{75}$ & AP$^\text{m}_\text{s}$ & AP$^\text{m}_\text{m}$ & AP$^\text{m}_\text{l}$ & \#Param. & OPs  \\

            \midrule DeiT-T 
            & 41.11 & \textbf{62.24} & 44.71 & 23.50 & 44.00 & \textbf{56.56}   
            & 37.26 & 59.20 & 39.63 & 17.35 & 39.55 & \textbf{56.62} & 28M & 338G  \\
            SageAttention-T 
            & 41.10 & 62.21 & 44.65 & 23.52 & 44.02 & 56.43   
            & 37.26 & \textbf{59.23} & 39.62 & 17.39 & 39.57 & 56.61 & 28M & 327G  \\
            \rowcolor{blue!5} BinaryAttention-T
            & \textbf{41.29} & 62.04 & \textbf{44.91} & \textbf{23.99} & \textbf{44.62} & 55.81    
            & \textbf{37.26} & 59.10 & \textbf{39.68} & \textbf{17.71} & \textbf{39.89} & 55.85 & 28M  & 313G \\
            
            \hline DeiT-S 
            & 45.49 & 67.29 & 49.10 & 28.91 & 48.56 & 61.57 
            & 40.87 & 63.90 & 43.57 & 21.95 & 43.56 & \textbf{60.89} & 44M & 440G  \\
            SageAttention-S 
            & 45.48 & 67.23 & 49.13 & 28.87 & 48.65 & 61.62 
            & 40.85 & 63.90 & 43.59 & 21.92 & 43.52 & 60.86 & 44M & 418G  \\
            \rowcolor{blue!5} BinaryAttention-S
            & \textbf{45.86} & \textbf{67.29} & \textbf{49.77} & \textbf{30.03} & \textbf{49.16} & \textbf{61.72}    
            & \textbf{41.01} & \textbf{64.10} & \textbf{43.68} & \textbf{22.39} & \textbf{43.71} & 60.34 & 44M & 390G   \\

            \hline DeiT-B 
            & 47.99 & 69.46 & 52.07 & 31.59 & 51.11 & \textbf{64.04}    
            & 42.98 & 66.48 & 46.29 & 23.82 & 45.64 & \textbf{61.85} & 111M & 785G  \\
            SageAttention-B 
            & 47.97 & 69.46 & 52.08 & 31.64 & 51.13 & 63.97    
            & 42.99 & 66.50 & 46.29 & 23.82 & 45.65 & 61.83 & 111M & 742G  \\
            \rowcolor{blue!5} BinaryAttention-B
            & \textbf{48.28} & \textbf{69.98} & \textbf{52.58} & \textbf{31.96} & \textbf{51.68} & 63.44    
            & \textbf{43.24} & \textbf{66.75} & \textbf{46.44} & \textbf{24.60} & \textbf{46.10} & 61.03 & 111M & 685G   \\

            \toprule
    
            \multicolumn{15}{c}{\textbf{(b) Cascade Mask R-CNN}} \\
            Backbone & AP$^\text{b}$ & AP$^\text{b}_\text{50}$ & AP$^\text{b}_\text{75}$ & AP$^\text{b}_\text{s}$ & AP$^\text{b}_\text{m}$ & AP$^\text{b}_\text{l}$ & AP$^\text{m}$ & AP$^\text{m}_\text{50}$ & AP$^\text{m}_\text{75}$ & AP$^\text{m}_\text{s}$ & AP$^\text{m}_\text{m}$ & AP$^\text{m}_\text{l}$ & \#Param. & OPs \\

            \midrule DeiT-T 
            & 46.39 & 64.58 & 50.11 & 27.05 & 50.03 & \textbf{63.71}   
            & 39.90 & 61.75 & 42.91 & 19.37 & 42.75 & \textbf{59.73} & 58M  & 595G  \\
            SageAttention-T 
            & 46.45 & 64.59 & 50.10 & 27.11 & 50.09 & 63.70   
            & 39.91 & 61.70 & 42.90 & 19.39 & 42.73 & 59.68 & 58M & 584G   \\
            \rowcolor{blue!5} BinaryAttention-T
            & \textbf{46.64} & \textbf{64.79} & \textbf{50.37} & \textbf{28.07} & \textbf{50.24}& 63.14    
            & \textbf{40.12} & \textbf{62.00} & \textbf{43.05} & \textbf{20.54} & \textbf{42.94} & 59.46 & 58M & 570G    \\

            \hline DeiT-S 
            & 49.44 & 68.10 & 53.16 & 32.01 & 53.20 & 65.18   
            & 42.52 & 65.12 & 45.90 & 23.32 & \textbf{45.40} & 60.89 & 75M & 696G  \\
            SageAttention-S 
            & 49.46 & \textbf{68.16} & 53.20 & 32.07 & \textbf{53.22} & 65.14   
            & 42.51 & 65.13 & 45.93 & 23.41 & 45.38 & 60.88 & 75M & 675G  \\
            \rowcolor{blue!5} BinaryAttention-S
            & \textbf{49.63} & 68.00 & \textbf{53.65} & \textbf{32.83} & 53.00 & \textbf{65.70}    
            & \textbf{42.72} & \textbf{65.42} & \textbf{46.09} & \textbf{23.80} & 45.23 & \textbf{61.37} & 75M & 647G   \\

            \hline DeiT-B 
            & \textbf{50.21} & 68.63 & 54.53 & 33.29 & \textbf{53.43} & 65.56    
            & 43.49 & 66.25 & 47.05 & \textbf{24.69} & 46.03 & 61.05 & 141M & 1041G  \\ 
            SageAttention-B 
            & 50.20 & 68.62 & \textbf{54.55} & 33.29 & 53.39 & 65.58    
            & 43.48 & 66.27 & 47.06 & 24.55 & 45.98 & 61.06 & 141M & 999G   \\  
            \rowcolor{blue!5} BinaryAttention-B
            & 50.16 & \textbf{68.80} & 54.09 & \textbf{32.90} & 53.31 & \textbf{66.26}    
            & \textbf{43.49} & \textbf{66.28} & \textbf{47.15} & 24.28 & \textbf{46.17} & \textbf{62.23} & 141M & 941G    \\
            \bottomrule
        \end{tabular}}
    }
    \caption{Comparison of object detection and instance segmentation on COCO. OPs are computed with input resolution of 1024$\times$1024. }
    \vspace{-2mm}
    \label{tab:detection}
\end{table*}

\begin{table}[!t]
    \centering
    \setlength{\tabcolsep}{1.2mm}
    \resizebox{0.9\linewidth}{!}{
        \begin{tabular}{l|c|cc|cc}
        \toprule
        Backbone & \begin{tabular}[c]{@{}c@{}} Crop \\ size \end{tabular} & \begin{tabular}[c]{@{}c@{}} mIoU \\ (SS) \end{tabular} &  \begin{tabular}[c]{@{}c@{}} mIoU \\ (MS) \end{tabular}  & \#Param. & OPs \\ 
        \midrule
        DeiT-T                               & $\text{512}^\text{2}$ & 39.82 & 40.68 & 11M & 227G  \\ 
        SageAttention-T                      & $\text{512}^\text{2}$ & 39.82 & 40.68 & 11M & 198G  \\ 
        \rowcolor{blue!5} BinaryAttention-T  & $\text{512}^\text{2}$ & \textbf{39.93} & \textbf{40.89} & 11M & 159G  \\
        \midrule
        DeiT-S                               & $\text{512}^\text{2}$ & 44.67 & 46.01 & 43M & 744G  \\ 
        SageAttention-S                      & $\text{512}^\text{2}$ & 44.67 & \textbf{46.01} & 43M & 687G  \\ 
        \rowcolor{blue!5} BinaryAttention-S  & $\text{512}^\text{2}$ & \textbf{44.75} & 45.95 & 43M & 610G  \\
        \midrule
        DeiT-B                               & $\text{512}^\text{2}$ & 46.86 & 47.74 & 166M & 2654G  \\ 
        SageAttention-B                      & $\text{512}^\text{2}$ & 46.86 & 47.74 & 166M & 2539G  \\ 
        \rowcolor{blue!5} BinaryAttention-B  & $\text{512}^\text{2}$ & \textbf{47.76} & \textbf{48.37} & 166M & 2384G  \\
        \bottomrule
        \end{tabular}
    }
    \caption{Comparison of semantic segmentation on ADE20K. `SS' and `MS' represent single-scale and multi-scale testing, respectively. OPs are calculated with input resolution of $512 \times 2048$. }
    \label{tab:seg}
    \vspace{-4mm}
\end{table}

\begin{table*}[!t]
    \centering
    \setlength{\tabcolsep}{1.2mm}
    \resizebox{0.78\linewidth}{!}{
    \begin{tabular}{lcc|ccccc|ccccc}
    \toprule
     Method & OPs& Steps & FID$\downarrow$   & sFID$\downarrow$  & IS$\uparrow$     & Pre.$\uparrow$ & Re.$\uparrow$  & FID$\downarrow$   & sFID$\downarrow$  & IS$\uparrow$     & Pre.$\uparrow$ & Re.$\uparrow$ \\
    \midrule
    \multicolumn{3}{c|}{(cfg=1.25)} & \multicolumn{5}{c|}{DiT-S/2} & \multicolumn{5}{c}{SiT-S/2} \\
    FlashAttention2 & 6.1G & 400K & 54.73 & \textbf{10.21} & 26.83  & 0.42 & 0.57 
                                  & 46.33 & \textbf{7.89}  & 31.99 & 0.46 & 0.59 \\
    SageAttention   & 5.8G & 400K & 54.74 & 10.21 & 26.86 & 0.42 & 0.57 
                                  & 46.32 & 7.89 & 32.00 & 0.46 & 0.59 \\
    \rowcolor{blue!5}
    BinaryAttention & 5.5G & 200K & \textbf{49.76} & 10.40  & \textbf{30.14} & \textbf{0.45} & \textbf{0.57}
                                  & \textbf{41.40} & 7.99 & \textbf{36.49} & \textbf{0.48} & \textbf{0.59} \\

    \midrule
    
    \multicolumn{3}{c|}{(cfg=1.50)} & \multicolumn{5}{c|}{DiT-S/2} & \multicolumn{5}{c}{SiT-S/2} \\

    FlashAttention2 & 6.1G & 400K & 43.87 & \textbf{8.82 } & 35.16  & 0.48 & 0.55 
                                  & 36.26 & \textbf{7.09}  & 42.24  & 0.52 & 0.56 \\
    SageAttention   & 5.8G & 400K & 43.90 & 8.82 & 35.19 & 0.48 & 0.55 
                                  & 36.25 & 7.09 & 42.25 & 0.52 & 0.56 \\
    \rowcolor{blue!5}
    BinaryAttention & 5.5G & 200K & \textbf{38.96} & 8.95  & \textbf{40.12} & \textbf{0.50} & \textbf{0.56} 
                                  & \textbf{31.18} & 7.12  & \textbf{50.54} & \textbf{0.55} & \textbf{0.56} \\

    \bottomrule
    \toprule
    
    \multicolumn{3}{c|}{(cfg=1.25)} & \multicolumn{5}{c|}{DiT-XL/2} & \multicolumn{5}{c}{SiT-XL/2} \\

    FlashAttention2 & 118.6G & 7000K & \textbf{3.22}  & \textbf{5.28}  & \textbf{201.77} & \textbf{0.76} & 0.62 
                                     & \textbf{3.62}  & \textbf{5.23}  & \textbf{193.51} & \textbf{0.75} & 0.64 \\
    SageAttention   & 117.1G & 7000K & 3.24  & 5.29  & 201.72 & 0.76 & 0.62
                                     & 3.68  & 5.25  & 191.57 & 0.75 & 0.64 \\
    \rowcolor{blue!5}
    BinaryAttention & 115.0G & 4000K & 4.26 & 6.33 & 199.59 & 0.72 & \textbf{0.65}
                                     & 4.22 & 5.74 & 191.57 & 0.74 & \textbf{0.65} \\

    \midrule
    
    \multicolumn{3}{c|}{(cfg=1.50)} & \multicolumn{5}{c|}{DiT-XL/2} & \multicolumn{5}{c}{SiT-XL/2} \\

    FlashAttention2 & 118.6G & 7000K & 2.27  & 4.60  & \textbf{278.24} & \textbf{0.83} & 0.57
                                     & \textbf{2.15}  & \textbf{4.60}  & \textbf{258.09} & \textbf{0.81} & 0.60 \\
    SageAttention   & 117.1G & 7000K & 2.27 & \textbf{4.59} & 278.03 & 0.83 & 0.58  
                                     & 2.16 & 4.63 & 256.07 & 0.80 & 0.61 \\
    \rowcolor{blue!5}
    BinaryAttention & 115.0G & 4000K & \textbf{2.19}  & 5.01 & 278.03 & 0.80 & \textbf{0.61} 
                                     & 2.21 & 4.89 & 250.02 & 0.79 & \textbf{0.61} \\
    \bottomrule
    \end{tabular}
    }
    \caption{Comparison of class-conditional image generation on ImageNet 256$\times$256. } 
    \vspace{-2mm}
    \label{tab:diffusion}
\end{table*}

\vspace{+1mm}
\noindent\textbf{Results.} Tab.~\ref{tab:classification} presents the comparison result. BinaryAttention demonstrates consistent performance advantages in multiple dimensions. 
Compared with the baseline DeiT (implemented with FlashAttention2) and the recent SageAttention, BinaryAttention achieves higher accuracy with reduced computational cost. For instance, BinaryAttention-T attains a top-1 accuracy of 72.88\%, outperforming DeiT-T by 0.68\% and SageAttention-T by 0.77\% with fewer OPs. BinaryAttention-S and BinaryAttention-B achieve top-1 accuracies of 80.24\% and 82.04\% at a resolution of 224$\times$224, surpassing SageAttention by 0.42\% and 0.21\%, respectively. Notably, at higher resolution 384$\times$384, BinaryAttention-B achieves the highest accuracy of 83.64\% with 50.2G OPs, demonstrating superior performance to DeiT-B (83.1\%, 55.4G OPs) and SageAttention-B (82.89\%, 53.2G OPs). 

Compared with W8A8 quantization methods, BinaryAttention delivers superior accuracy albeit requires more OPs.  BinaryAttention-T costs 1.1G OPs, higher than PTQ4ViT-T (0.3G OPs), but achieves much higher accuracy (72.88\% vs. 71.6\%). Note that BinaryAttention can be seamlessly integrated with these quantization techniques for linear layers. Both BinaryAttention and SageAttention can achieve a substantial reduction in OPs when combined with PTQ4ViT, but BinaryAttention maintains high accuracy. Specifically, BinaryAttention-B+PTQ4ViT holds a top-1 accuracy of 83.55\% with only 13.5G OPs, which is 0.57\% higher than SageAttention-B+PTQ4ViT with 3G fewer OPs. 

We also include Linear Transformers and SSMs in the comparison just for reference. While Flatten-T, InLine-T, and Vim-T achieve  higher accuracy in some configurations, they struggle to scale up effectively. InLine-S/B relies on larger input and Vim-S/B requires significantly more parameters and OPs. In contrast, BinaryAttention maintains robust performance without architecture changes. 

The comprehensive evaluations presented above highlight the effectiveness of BinaryAttention in maintaining the expressive capability of standard attention while achieving substantial efficiency improvement.

\subsection{Object Detection and Instance Segmentation}

\textbf{Settings.} We then evaluate object detection and instance segmentation tasks with COCO 2017 dataset \cite{lin2014microsoft} and Detectron2 \cite{wu2019detectron2} library. Following ViTDet \cite{li2022exploring}, we apply Mask R-CNN \cite{he2017mask} and Cascade Mask R-CNN \cite{cai2018cascade} as detector heads, with ImageNet-1K pre-trained BinaryAttention-T/S/B as backbones. During training, we employ the AdamW optimizer with beta set to (0.9, 0.999), momentum of 0.9 and a batch size of 64. The initial learning rate is set to 0.001 with a weight decay of 0.1. A linear learning rate decay is adopted with a warm-up of 250 iterations. 

\vspace{+1mm}
\noindent\textbf{Results}. The results are reported in Tab.~\ref{tab:detection}. 
BinaryAttention shows competitive performance with lower computational cost. 
For the Mask R-CNN head, BinaryAttention-T achieves 0.18 higher box mAP than DeiT-T with the same mask mAP of 37.26, while reducing OPs by 25G compared to DeiT-T and 14G to SageAttention-T. BinaryAttention-S exceeds DeiT-S by 0.37 in box mAP and 0.14 in mask mAP, with significant gains on small-size objects, from 28.87 to 30.03 in small box mAP. BinaryAttention-B maintains the superior performance to DeiT-B, while performing slightly lower in large-size objects. 
Similar observations can be made in the Cascade Mask R-CNN head. BinaryAttention-T and -S achieve consistent gains over their DeiT counterparts, with -S reaching a box mAP of 49.63 and a mask mAP of 42.72. Furthermore, BinaryAttention-B delivers a more favorable accuracy and efficiency trade-off, matching its full-precision counterpart with a 10\% reduction in OPs. 

\subsection{Semantic Segmentation}

\textbf{Settings.} We further evaluate BinaryAttention on semantic segmentation using ADE20K \cite{zhou2019semantic} and MMSegmentation \cite{mmseg2020}. Following \cite{fang2024eva}, we adopt the widely used UPerNet \cite{xiao2018unified} as the segmenter and our pre-trained models as backbones. We train the models for 60K iterations with a batch size of 32, using the AdamW optimizer with a weight decay of 0.01, and an initial learning rate of $\text{6}\times\text{10}^\text{-5}$ with a linear learning rate decay, following a warm-up of 1500 iterations. 

\begin{figure}[!t]
    % \centering
    \includegraphics[width=\columnwidth]{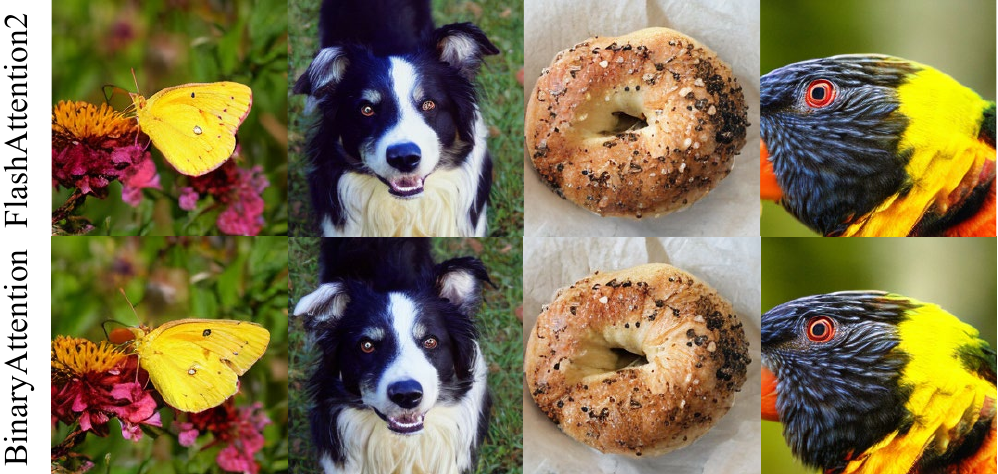}
    \caption{Qualitative comparison of generated image by DiT-XL/2 (cfg=1.50) using FlashAttention2 and BinaryAttention. }
    \label{fig:qualitative comparison}
    \vspace{-4mm}
\end{figure}

\vspace{+1mm}
\noindent\textbf{Results.} As shown in Tab.~\ref{tab:seg}, BinaryAttention achieves impressive performance on the semantic segmentation task. For instance, BinaryAttention-T achieves a single-scale mIoU of 39.93 and a multi-scale mIoU of 40.89 with fewer OPs, outperforming DeiT-T by 0.11 and 0.21 mIoU, respectively. BinaryAttention-S is slightly higher than DeiT-S in single-scale mIoU and nearly identical to the baseline in multi-scale testing. BinaryAttention-B delivers a considerable improvement, achieving 47.76 single-scale mIoU and 48.37 multi-scale mIoU, exceeding DeiT-B by 0.90 and 0.63 mIoU, respectively, while reducing computational cost by 270G OPs. These results confirm that BinaryAttention excels at complex scene understanding while preserving the fine-grained details essential for semantic segmentation.

\subsection{Image Generation}
\label{sec:image generation}

\textbf{Settings.} Finally, we explore the applicability of BinaryAttention to class-conditional image generation task. Following Diffusion Transformers (DiT) \cite{peebles2023scalable} and Scalable Interpolant Transformers (SiT) \cite{ma2024sit}, we replace their standard attention modules with our BinaryAttention. We train all models on ImageNet at $\text{256}\times\text{256}$ image resolution, with 200K iterations for small-size models and 4000K for XL variants, using the original training configurations without any hyperparameter changes. The models are initialized with the pre-trained DiT and SiT models. During evaluation, we generate 50K images using 250 sampling steps, applying the DDPM scheduler to DiT and the ODE solver to SiT. Fr\'echet Inception Distance (FID) \cite{heusel2017gans}, sliding FID (sFID) \cite{nash2021generating}, Inception Score (IS) \cite{salimans2016improved} and Precision and Recall \cite{kynkaanniemi2019improved} are used to evaluate the generative performance. 

\vspace{+1mm}
\noindent\textbf{Results.} The results are summarized in Tab.~\ref{tab:diffusion}. For DiT-S/2 and SiT-S/2 models, BinaryAttention shows substantial improvements over both FlashAttention2 and SageAttention while having fewer OPs. It achieves significantly better FID with different classifier-free guidance (cfg) scales than the baseline models at 400K steps. It also demonstrates higher IS and Precision. For XL models, BinaryAttention shows slightly lower performance when using a smaller cfg scale at 4000K steps, but it matches or even exceeds FlashAttention2 and SageAttention with a higher cfg (1.50), achieving the lowest FID of 2.19 for DiT-XL/2. Qualitative comparison in Fig.~\ref{fig:qualitative comparison} further demonstrates that BinaryAttention achieves a generation quality on par with full-precision models, producing images that are detailed and structurally consistent.

\subsection{Ablation Study}
\label{sec:ablation}

\begin{table}[!t]
    \centering
    \setlength{\tabcolsep}{1.2mm}
    \resizebox{0.84\linewidth}{!}{
        \begin{tabular}{ccc|ccc}
        \toprule
        \multirow{2}*{Scale} & \multirow{2}*{Bias} & \multirow{2}*{Distillation} & DeiT-T  & DeiT-S & DeiT-B \\
        &  &  & \multicolumn{3}{c}{Top-1 Accuracy}   \\ 
        \midrule
        \multicolumn{3}{c|}{Baseline (full-precision)} & 72.2 & 79.8 & 81.8 \\
        \midrule
        \ding{55} & \ding{55} & \ding{55} & 71.95 & 79.59 & 81.10\\
        \ding{51} & \ding{55} & \ding{55} & 72.42 & 79.81 & 81.33 \\
        \ding{51} & \ding{55} & \ding{51} & 72.44 & 79.97 & 81.99 \\
        \ding{51} & \ding{51} & \ding{51} & 72.88 & 80.24 & 82.04 \\

        \bottomrule
        \end{tabular}
    }
    \caption{Ablation studies on BinaryAttention for the scaled binary representations, bias enhancement and self-distillation strategy on ImageNet-1K benchmark, using DeiT architectures. }
    \label{tab:ablation:1}

\end{table}

We first conduct a series of ablation studies to analyze the core components of BinaryAttention, including scaled binary representations, bias enhancement, and the self-distillation strategy. The experiments are performed on the ImageNet-1K \cite{deng2009imagenet} benchmark by using DeiT \cite{touvron2021training} architectures. The results are summarized in Tab.~\ref{tab:ablation:1}.

\vspace{+1mm}
\noindent \textbf{Scaled Binary Representations.} We evaluate the role of scaling factors in binary representations. 
When scaling is not applied, BinaryAttention shows performance drop across all models, with top-1 accuracy decreased by $-$0.25\%, $-$0.21\%, and $-$0.70\% for DeiT-T, -S, and -B, respectively. 
Introducing scaling factors effectively solves this issue by minimizing the quantization error, with DeiT-T even exceeding its full-precision baseline (72.42\% vs. 72.2\%), demonstrating that proper scaling is essential for preserving representational capabilities in binary space.

\vspace{+1mm}
\noindent \textbf{Bias Enhancement.} We simply employ a learnable relative position bias \cite{liu2021swin} as the bias term, which exhibits distinct effects across model scales. It provides an accuracy gain of 0.44\% and 0.27\%  for DeiT-T and -S, respectively, while offering a slight improvement for DeiT-B, from 81.99\% to 82.04\%. This discrepancy stems from the relationship between model capacity and the expressive power of binary representations. For smaller models, the limited dimension constrains the diversity of attention patterns, making them more susceptible to distribution collapse. The bias term effectively mitigates this by introducing additional contextual or structural information. For larger models, higher-dimensional binary representations naturally preserve richer similarity structures, yielding more modest gains.

\vspace{+1mm}
\noindent \textbf{Self-distillation Strategy.} We investigate the role of self-distillation, which  slightly improves DeiT-T and -S models but significantly boosts the accuracy of DeiT-B by 0.66\%. This improvement suggests that self-distillation effectively counteracts the distribution shift introduced by quantization errors while encouraging sign-aligned similarity between binary representations and its full-precision counterparts.

\begin{table}[!t]
    \centering
    \setlength{\tabcolsep}{1.2mm}
    \resizebox{0.82\linewidth}{!}{
        \begin{tabular}{lcccc}
        \toprule
         & CosSim & Relative L1  & RMSE  & Precision  \\

        \midrule
        Layer (0)  & 0.9186 & 0.4840 & 0.4165 & 0.7716 \\
        Layer (6)  & 0.8740 & 0.7353 & 0.5084 & 0.7301  \\
        \bottomrule
        \end{tabular}
    }
    \caption{Attention pattern comparison between FlashAttention2 and BinaryAttention by DeiT-B on ImageNet-1K validation set.}
    \label{tab:ab:patt}
    \vspace{-4mm}
\end{table}

\begin{table}[!t]
    \centering
    \setlength{\tabcolsep}{1.2mm}
    \resizebox{0.94\linewidth}{!}{
        \begin{tabular}{lc|cc}
        \toprule
        Method &   Top-1 & Mem. (512) & Mem. (1024)  \\
         \midrule
        FlashAttention2        & 72.2   & 1705M & 5304M \\

        SageAttention          & 72.11  & 1705M & 5304M \\
        
        BinaryAttention (den.) & 72.88  & 3246M & 29904M \\

        BinaryAttention (dec.) & 72.97  & 1706M & 5307M \\
        \bottomrule
        \end{tabular}
    }
    \caption{Memory comparison by DeiT-T using FlashAttention2, SageAttention and BinaryAttention at resolutions of 512 and 1024.}
    \label{tab:ab:mem}
\end{table}

\begin{table}[!t]
    \centering
    \setlength{\tabcolsep}{1.2mm}
    \resizebox{0.99\linewidth}{!}{
        \begin{tabular}{ccccc}
        \toprule
        FlashAttention2  & SageAttention  & BinaryAttention  & \begin{tabular}[c]{@{}c@{}} Quant \\ Q\&K \end{tabular}  & \begin{tabular}[c]{@{}c@{}} Quant \\ V \end{tabular}  \\
        \midrule
        175.3ms & 124.6ms & 88.2ms & 2.8ms & 1.9ms \\
        \bottomrule
        \end{tabular}
    }
    \caption{Latency of attention kernels and quantization components measured on A100 GPUs.}
    \label{tab:ab:quant}
    \vspace{-4mm}
\end{table}

\vspace{+1mm}
\noindent \textbf{Attention Pattern Fidelity.}
We further analyze whether BinaryAttention preserves the original attention dynamics. We use Cosine Similarity, Relative L1 Distance, RMSE, and Precision as evaluation metrics, where Precision measures the accuracy of matching the top 100 most attended tokens. As shown in Tab.~\ref{tab:ab:patt}, BinaryAttention maintains high consistency with full-precision attention on ImageNet-1K validation set, with cosine similarity above 0.87 and precision around 0.75, demonstrating that BinaryAttention effectively preserves key relational patterns and structural relationships.

\vspace{+1mm}
\noindent \textbf{Memory and Quantization Overhead.} 
Finally, we report the memory footprint in Tab.~\ref{tab:ab:mem} and the quantization overhead in Tab.~\ref{tab:ab:quant}. BinaryAttention incurs extra memory primarily from the bias term. With a dense bias, memory grows rapidly with resolution, and with a decomposable bias, \eg, a sum over spatial directions, the overhead becomes almost negligible. Meanwhile, the quantization cost is modest, requiring 2.8ms for query and key and 1.9ms for value (4.7ms in total), which accounts for about 5\% of the BinaryAttention kernel.

\section{Conclusion}
\label{sec:conclusion}

We presented BinaryAttention, a simple yet accurate and efficient 1-bit qk-attention for vision and diffusion transformers. By establishing theoretical guaranties that token similarity persists in binary space, we incorporated scaled binary representations, bias enhancement, and hybrid quantization into standard attention, achieving 2$\times$ speedup over FlashAttention2. Extensive experiments on image classification, detection, segmentation and generation validated that BinaryAttention matched or even surpassed full-precision attention with only 1-bit representations, demonstrating its strong potential for implementing ultra-low-precision inference deployment in practical vision tasks without compromising performance. 

\vspace{+1mm}
\noindent\textbf{Limitations.} Despite significant acceleration for computing $\bm{Q}\bm{K}^T$ similarity, the $\bm{P}\bm{V}$ multiplications employ more conservative quantization, leaving room for greater end-to-end efficiency. Furthermore, our method currently focuses on optimizing attention computations specifically, leaving the complementary potential of jointly quantizing other components like MLP layers for future investigation.

{
    \small
    \bibliographystyle{ieeenat_fullname}
    \bibliography{reference}
}


\twocolumn[
    \centering
    \Large
    \textbf{\thetitle}\\
    \vspace{0.5em}Supplementary Material \\
    \vspace{1.0em}
] %< twocolumn

\noindent In this supplementary file, we provide following materials:

\vspace{+1mm}
\begin{itemize}[leftmargin=1.0em]
\item[\ref{sec:suppl:proof}] Proof of Theorem 1 (referring to Sec.~\ref{sec:theoreticalmotivation} in the main paper);
\item[\ref{sec:suppl:alg}] Algorithm of BinaryAttention (referring to Sec.~\ref{hardware} in the main paper); 
% \item[\ref{sec:suppl:ablation}] Detailed ablation studies of BinaryAttention (referring to Sec. 5 in the main paper); 
\item[\ref{sec:suppl:setting}] Experimental details in classification (referring to Sec.~\ref{sec:efficiency} in the main paper); 
\item[\ref{sec:suppl:visual}] More qualitative comparisons (referring to Sec.~\ref{sec:image generation} in the main paper).
\end{itemize}

\appendix
\section{Proof of Theorem 1}
\label{sec:suppl:proof}

\begin{proof}
Consider the element  $(i,j)$ of the matrix $\bm{s}\bm{t}^T$:
$$
[\bm{s}\bm{t}^T]_{ij}=\bm{s}_i\bm{t}_j=\text{sign}(\bm{q}_i)\text{sign}(\bm{k}_j).
$$
Since $\bm{q}$ and $\bm{k}$ are assumed to be jointly Gaussian with zero-mean, the pair $(\bm{q}_i,\bm{k}_j)$ is also jointly Gaussian. 
We have:
\begin{align*}
\Var(\bm{q}_i)=\bm{\Sigma}_{qq}[i,i], &\ \Var(\bm{k}_j)=\bm{\Sigma}_{kk}[j,j], \\
\Cov(\bm{q}_i,\bm{k}_j)&=\bm{\Sigma}_{qk}[i,j].
\end{align*}

Then the correlation of $\bm{q}_i$ and $\bm{k}_j$ is given by:
$$
\rho_{ij}=\frac{\bm{\Sigma}_{qk}[i,j]}{\sqrt{\bm{\Sigma}_{qq}[i,i]\bm{\Sigma}_{kk}[j,j]}}=\bm{C}_{ij}.
$$

Let $x=\bm{\Sigma}^{-\frac{1}{2}}_{qq}[i,i]\bm{q}_i$ and $y=\bm{\Sigma}^{-\frac{1}{2}}_{kk}[j,j]\bm{k}_j$, then the two variables $x,y$ are standard Gaussian with correlation $\rho_{ij}$, and the joint density can be expressed as:
$$
p(x,y)=\frac{1}{2\pi\sqrt{1-\rho_{ij}}}\text{exp}\left(-\frac{x^2-2\rho_{ij}xy+y^2}{2(1-\rho_{ij}^2)}\right).
$$
We now calculate the expectation of $\text{sign}(x)\text{sign}(y)$, where
\[
\text{sign}(x)\text{sign}(y) =
\begin{cases}
+1 & \text{if } x\geq 0, y\geq0 \ \text{or }\ x \leq 0, y \leq 0 \\
-1 & \text{if } x\geq 0, y < 0 \ \text{or }\ x < 0, y \geq 0 \\
\end{cases}.
\]

By the symmetry of standard Gaussian, we have
$$
\mathbb{E}[\text{sign}(x)\text{sign}(y)]=4\mathbb{P}(x\geq0,y\geq 0)-1.
$$
Since
{\small
\begin{align*}
&\mathbb{P}(x\geq0,y\geq 0)=\int_{0}^{\infty}\int_{0}^{\infty} p(x,y)dxdy  \xLeftrightarrow[y=R\sin{\theta}]{x=R\cos{\theta}}\\
& \frac{1}{2\pi\sqrt{1-\rho_{ij}}}\int_{0}^{\frac{\pi}{2}}\int_{0}^{\infty}\text{exp}\left( -\frac{R^2(1-\rho_{ij}\sin{2\theta})}{2(1-\rho_{ij}^2)}\right)RdRd\theta \\
&=\frac{1}{2\pi}\arcsin{\rho_{ij}}+\frac{1}{4},
\end{align*}
}
we have: 
$$
\mathbb{E}[\text{sign}(x)\text{sign}(y)]=\frac{2}{\pi}\arcsin{\rho_{ij}}=\frac{2}{\pi}\arcsin{\bm{C}_{ij}}.
$$
Note that the $\text{sign}(\cdot)$ function is scale-invariant for any strictly positive scales, we can ensure that
$$
\mathbb{E}[\text{sign}(\bm{q}_i)\text{sign}(\bm{k}_j)] = \mathbb{E}[\text{sign}(x)\text{sign}(y)]=\frac{2}{\pi}\arcsin{\bm{C}_{ij}}. \\
$$
Since it holds for all $i,j=1,\dots,d$, there is:
$$
    \mathbb{E}[\bm{s}\bm{t}^T]=\frac{2}{\pi}\arcsin{\bm{C}}.
$$
\end{proof}

\vspace{-3mm}
\begin{algorithm}[!b]
    % \label{algorithm1}
    % \footnotesize 
    % \scriptsize
    % \vspace{-15pt}
    \small
    \caption{Implementation of BinaryAttention}
    % \SetAlgoCaptionLayout{small}
    \label{alg:binaryattn}
    \KwIn{Matrices $\bm{Q},\bm{K},\bm{V}\in\mathbb{R}^{N\times d}$, bias $\bm{B}\in\mathbb{R}^{N\times N}$, block size $B_r,B_v$ .}
    \KwOut{Matrices $\bm{O}\in\mathbb{R}^{N\times d}$ .}
    \textbf{Processing:} 
    $(\mu_q,\bm{S}) \gets \phi(\bm{Q}),(\mu_k,\bm{T}) \gets \phi(\bm{K}), (\delta_v,\tilde{\bm{V}}) \gets \psi(\bm{V})$ \Comment*[r]{Quantization by Eq.(5) and Eq.(7)}
    % Initialize $\bm{O}=(0)_{N\times d}$, $l=(0)_{N}$, $m=(-\infty)_{N}$ \; 
    Divide $\bm{S},\bm{O}$ into $ T_r:=\left \lceil {N}/{B_r} \right\rceil$ blocks $\{\bm{S}_i\}$ and $\{\bm{O}_i\}$ \;
    Divide $\bm{T},\tilde{\bm{V}}$ into $ T_v:=\left \lceil {N}/{B_v} \right\rceil$ blocks $\{\bm{T}_j\}$ and $\{\tilde{\bm{V}}_j\}$ \;
    Divide $\bm{B}$ into $ T_r\times T_v$ blocks $\{\bm{B}_{ij}\}$ \Comment*[l]{if bias}
    % \Comment*[r]{If bias}

    \For{$i=1$ \KwTo $T_r$}{
        Load block $\bm{S}_i$ from HBM to SRAM \; % of the GPU
        Initialize $\bm{O}_{i,0}=(0)_{B_r\times d}, l_{i,0}=(0)_{B_r}, m_{i,0}=(-\infty)_{B_r}$ \;
        \For{$j=1$ \KwTo $T_v$}{
            Load blocks $\bm{T}_j,\tilde{\bm{V}}_j,\bm{B}_{ij}$ from HBM to SRAM \;
            % $\bm{S}_{ij} \gets \textcolor{blue}{\text{BinaryMatmul}}(\bm{S}_i,\bm{T}_j)\times \mu_q \times \mu_k$ \;
            $\bm{S}_{ij} \gets {\text{BinaryMatmul}}(\bm{S}_i,\bm{T}_j)\times \mu_q \times \mu_k$ \;
            $\bm{S}_{ij} \gets \bm{S}_{ij} +\bm{B}_{ij}$ \Comment*[l]{if bias}
            $m_{ij} \gets \text{max}(m_{i,j-1},\text{rowmax}(\bm{S}_{ij}))$ \;
            $\hat{\bm{P}}_{ij} \gets \text{exp}(\bm{S}_{ij}-m_{ij})$ \;
            $l_{ij} \gets e^{m_{i,j-1}-m_{ij}}l_{i,j-1}+\text{rowsum}(\hat{\bm{P}}_{ij})$ \;
            % $\bm{O}_{ij} \gets \textcolor{blue}{\text{IntMatmul}}(\hat{\bm{P}}_{ij}\times255,\tilde{\bm{V}}_j)$\;
            $\bm{O}_{ij} \gets {\text{IntMatmul}}(\hat{\bm{P}}_{ij}\times255,\tilde{\bm{V}}_j)$ \;
            $\bm{O}_{ij} \gets \text{diag}(e^{m_{i,j-1}-m_{ij}})^{-1}\bm{O}_{i,j-1}+\bm{O}_{ij}$ \;
        }
        $\bm{O}_i \gets \text{diag}(l_{i,Tv})^{-1}\bm{O}_{i,T_v}/255\times \delta_v$ \;
        Write $\bm{O}_i$ \;

    }
    \Return $\bm{O}=\{\bm{O}_i\}$ .
    % \vspace{-15pt}
\end{algorithm}

\section{Algorithm of BinaryAttention}
\label{sec:suppl:alg}

Our implementation of BinaryAttention is built upon the fundamental principles of FlashAttention2 \cite{dao2023flashattention2} and SageAttention \cite{zhang2025sageattention} while introducing specialized optimizations for binary and low-precision computations. The complete algorithm is presented in Algorithm~\ref{alg:binaryattn}.

\section{Experimental Details in Classification}
\label{sec:suppl:setting}
\textbf{Settings.} We benchmark BinaryAttention on ImageNet-1K \cite{deng2009imagenet} dataset. Following the experimental configurations of DeiT \cite{touvron2021training}, we employ the AdamW optimizer for 300 epochs with beta set to (0.9, 0.999), momentum of 0.9 and a batch size of 1024. An initial learning rate of $\text{10}^{-\text{4}}$, a minimum learning rate of $\text{10}^{-\text{5}}$ and a weight decay of 0.02 are used. The learning rate follows a cosine annealing schedule with a warm-up of 5 epochs. We include commonly used augmentation and regularization strategies, consistent with the training of DeiT. The drop path rate is set to 0.1 for all BinaryAttention variants. Before training, the models are initialized with full-precision pre-trained weights. We utilize the self-distillation strategy with the full-precision counterpart as teacher, and implement quantization-aware training with Straight-Through Estimators (STE) \cite{bengio2013estimating}. For the input resolution of $\text{384}\times \text{384}$, we continue to fine-tune the models for 30 epochs, with a batch size of 512, a constant learning rate of $\text{10}^{-\text{5}}$, and a weight decay of $\text{10}^{-\text{8}}$.

\begin{table}[!ht]
    \centering
    % \vspace{-3.2mm}
    \setlength{\tabcolsep}{1.2mm}
    \resizebox{0.60\linewidth}{!}{
        \begin{tabular}{cccc}
        \toprule
        Epochs & DeiT-T  & DeiT-S  & DeiT-B  \\

        \midrule 
        Baseline & 72.2 & 79.8 & 81.8 \\
        \hline
        100  & 71.98 & 79.44 & 81.80  \\
        300  & 72.44 & 79.97 & 81.99   \\        
        \bottomrule
        \end{tabular}
    }
    \caption{Top-1 accuracy of DeiT models using BinaryAttention without bias at 100 and 300 fine-tuning epochs.}
    \label{tab:ab:epochs}
    \vspace{-2mm}
\end{table}

\begin{figure}[!b]
    % \centering
    \includegraphics[width=\columnwidth]{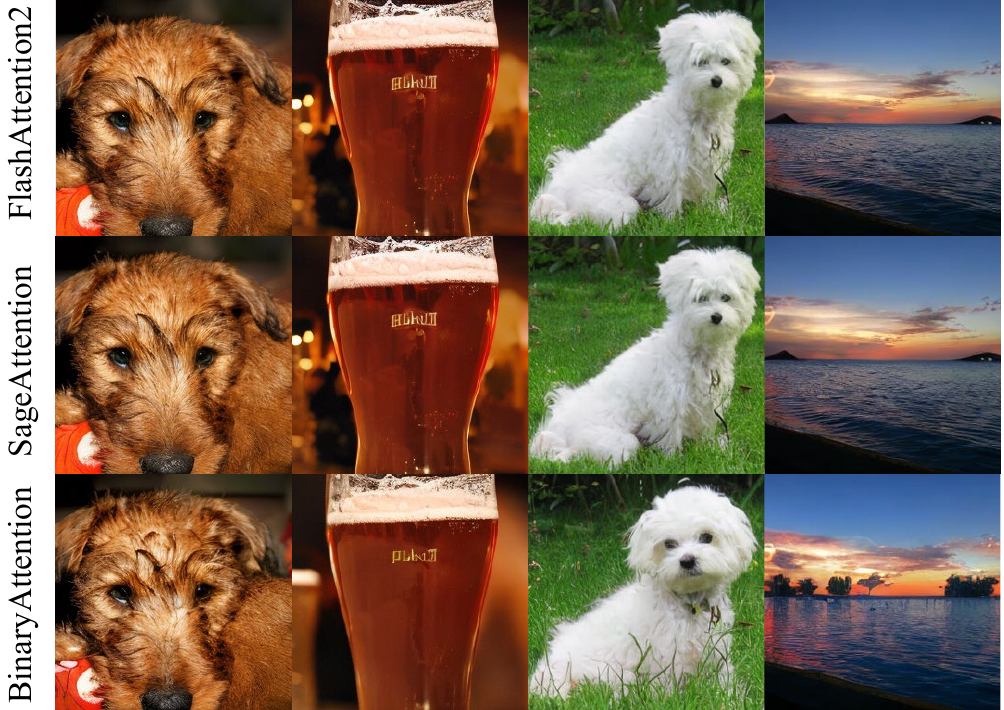}
    \caption{More qualitative comparisons of generated image by DiT-XL/2 (cfg=1.50) using FlashAttention2, SageAttention and BinaryAttention. }
    \label{fig:suppl:qualitative comparison}
\end{figure}

\vspace{+2mm}
\noindent \textbf{Tuning Cost.} In extreme low-bit quantization, fine-tuning is a standard and necessary step to bridge performance gaps. While training-free methods prioritize convenience, their performance is strictly limited by the baseline, whereas BinaryAttention raises the accuracy. For the best performance, we employ a fine-tuning schedule of 300 epochs, which is consistent with common practice in low-bit approaches such as BiViT \cite{he2023bivit}. We further report the performance at different fine-tuning epochs in Tab.~\ref{tab:ab:epochs}. It can be seen that with 100 fine-tuning epochs, BinaryAttention already nearly matches the baseline, with Top-1 accuracy gaps of 0.22/0.36/0.00 on DeiT-T/S/B, respectively. Extending fine-tuning to 300 epochs not only recovers the baseline performance but yields a modest improvement.

\section{More Qualitative Comparisons}
\label{sec:suppl:visual}

Fig.~\ref{fig:suppl:qualitative comparison} provides more qualitative comparisons of FlashAttention2, SageAttention, and BinaryAttention, showing additional images generated by the DiT-XL/2 model (cfg=1.50). 
We can see that SageAttention and FlashAttention2 produce nearly identical images, while BinaryAttention produces slightly different content but maintains competitive generation quality with sufficient details. 

\end{document}